\documentclass[preprint,authoryear,12pt]{elsarticle}

\usepackage{graphics}
\usepackage{graphicx}
 \usepackage{epsfig}
\usepackage{booktabs} 
\usepackage{array} 
\usepackage{paralist} 
\usepackage{verbatim} 
\usepackage{subfig} 
\usepackage{amsmath}

\usepackage{amssymb}

\newtheorem{theorem}{Theorem}
\newtheorem{proof}{Proof}

\journal{Nuclear Physics B}

\begin{document}

\begin{frontmatter}


 \author{A.Lotsi\corref{cor}}
 \ead{a.lotsi@rug.nl}
 
 \author{E.Wit}
 
 \ead[url]{http://www.math.rug.nl/stat/Main/HomePage}
 \cortext[cor]{Corresponding author} 
 \address{Department of Statistics and Probability, Johann Bernoulli Institute University of Groningen Nijenborgh 9, 9747 AG Groningen The Netherlands}

\title{High dimensional Sparse  Gaussian Graphical Mixture Model}




\begin{abstract}
This paper considers the problem of networks reconstruction from  heterogeneous data using a Gaussian Graphical Mixture Model (GGMM). It is well known that parameter estimation  in this context is challenging due to large numbers of variables coupled with the degenerate  nature  of the likelihood. We propose as a solution a penalized maximum likelihood technique by imposing an $l_{1}$ penalty on the precision matrix.  Our approach shrinks the parameters thereby resulting in better identifiability and variable selection.  We use the Expectation Maximization (EM) algorithm which involves  the graphical LASSO to estimate the mixing coefficients and the precision matrices. We  show that under certain regularity conditions  the Penalized Maximum Likelihood (PML) estimates are consistent.   We demonstrate the performance of the PML estimator through simulations and we show the utility of our method for high dimensional data analysis in a genomic application.

\end{abstract}


\end{frontmatter}



\section{Introduction}	
 Networks reconstruction has become an attractive paradigm of genomic science. Suppose we have data originate from different densities such as $\pi_{1}\mathcal{N}(\mu_{1},\Sigma_{1})$, $\pi_{2}\mathcal{N}(\mu_{2},\Sigma_{2})$,...\,$\pi_{K}\mathcal{N}(\mu_{K},\Sigma_{K})$,  where $\mathcal{N}(\mu,\Sigma)$ is a multivariate normal distribution with mean vector $\mu$ and variance covariance matrix $\Sigma$ and $\pi_{k}$s are the mixture proportions. The question we ask ourselves is what is the underlining networks from which the data come from? Statistical methods for analyzing such data are subject to active research currently \citep{AgakovOS12}. Gaussian graphical  Model (GGM) are a way to model such data.

 A Gaussian graphical  Model  for a random vector $Y=(Y_{1},...,Y_{p})$ is a pair $(G,P)$ where  $G$ is an undirected graph and $P = \left\{N(\mu,\Theta^{-1})\right\}$ is the model comprising all multivariate normal distributions whose inverse covariance matrix or precision matrix entries satisfies $(u,v)\in G  \Longleftrightarrow \Theta_{uv} \neq 0$. The conditional independence relationship among nodes are captured in $\Theta$.  Consequently, the problem of selecting the graph is equivalent to estimating the off-diagonal zero-pattern of the concentration matrix. Further details on these models as well interpretation of the conditional independency on the graph can be found in \citep{Lauritzen96}

In genomics, often there is heterogeneity in the data. We observe that in broad range of real world application ranging from finance to system biology, structural dependencies between the variables are rarely homogeneous i.e our population of individuals  may come  from  different clusters or mixture components without any information about their cluster membership. One challenge  is,  given only the sample measurement and with sparsity constraint, to recover the underlying networks.

Mixture distributions are often used to model heterogeneous data or observations supposed to have come from one of $K$ different components.  Under Gaussian mixtures, each component is suitably modelled by a family of Gaussian probability density. This paper deals with the problem of structural learning in reconstructing the underlying  graphical networks (using a graphical Gaussian model) from a data supposed to have come from a mixture of Gaussian distributions.

 We consider  model-based clustering \citep{McLachlan01032002} and assume that the data come from a finite mixture model where each component represents  a cluster. A large  literature  exists in normal mixture models; \citep{Lo01102001,Bozdogan}. Our focus here is on a high dimensional data setting where we present an algorithm based on a regularized  expectation maximization using Gaussian Mixture Model (GMM). We assume that our data $\mathbf{Y}_{i}=\left(Y_{i1},...,Y_{ip}\right)^{'}$ is generated through a $K \geq 1$ latent generative mixture components. We aim to group the data into a few $K$ clusters  and identify which observations are from which Gaussian components. 

A natural way for parameter estimation in GMMs is via a maximum likelihood estimation. However some performance degradation is encountered owing to the identifiability of the likelihood and the high dimensional setting. To overcome these problems, \cite{1993}   proposed a parameter reduction technique by re-parameterizing the covariance matrix through eigenvalue decomposition. In doing so,  some parameters are shared across clusters.  As a result of a continuous increasing number of dimensions, this approach can not totally alleviate the $(n <<p)$ phenomena.  Recently proposals to overcome the high dimensionality problem involve  estimating sparse precision matrix. Among these proposals is the penalized  log likelihood technique of \cite{Friedman01072008}, an $l_{1}$ regularization approach which encourages many of the entries of the precision matrix to be $0$. Our method is based on this idea. The $l_{1}$  penalty promotes sparsity. We provide sufficient conditions for consistency  of the penalized MLE.

Closely related to our work is that of \cite{Pan2007} where  variable selection is considered in model-based clustering. They considered GMM and penalize only the mean vectors and seeking to estimate sparse mean vectors. They assumed a common diagonal covariance matrix for all clusters. This work was later extended to \citep{Pan2009} where a new approach to penalized model-based clustering was considered but this time with unconstrained covariance matrices. However not much has been said about the consistency of the resulting estimators. Another recent work in this field is the work  by \cite{AgakovOS12} that learn structures of sparse high dimension latent variables with application to mixtures.
 
	In this article, we propose a penalized likelihood approach in the context of Gaussian Graphical mixture model, which  constraints the cluster distributions to be sparse. The parameters in the cluster distributions are estimated by incorporating an existing  graphical lasso method for covariance estimation into an  EM algorithm. In effect, we view each cluster as an instance of a particular GGM.  Therefore we aim at not only identify the population of individuals cluster membership but also  the dependencies  among the variables in each subgroup.  Additionally, we assess how well the resultant graphs  obtained through Glasso relate to the true graphs and we provide consistency results of  the  estimates.  Throughout this paper, we assume $K$, the number of components of mixture models is known.

The reminder of this article is organized as follows: We introduce the model, set up the  PMLE approach  and summarize the main result in connection with the consistency of the Glasso estimator in section 2. We then proceed with the inference procedure through a penalized version of the EM algorithm in section 3. In section 4 we present some simulations and an example of applications to illustrate our results. We conclude with a brief discussion and  future work in section 5.

\section{Penalized maximum likelihood estimation}
In this section we introduce our model-based clustering with GGM, then we derive the penalized  likelihood upon which statistical inference via the EM algorithm is based and prove consistency of the PMLE.

\subsection{The Mixture model}

The model consists of assuming that a variable $Z_{i}$,   describing which component an individual originates,  is a multinomial random variable with parameters $\pi_{k}$ denoting the mixture proportions or the mixing coefficients with $(0<\pi_{k}<1)$,  $\sum^{K}_{k=1}\pi_{k}=1$, and $K$ is known. In  essence
$$P(Z_{i}=k)=\pi_{k}$$ In our mixture model, we suppose that some vector-valued random variables $\mathbf{Y}_{1},...,\mathbf{Y}_{n}$ are a  random sample from the $K$ mixture components.  We  model each subpopulation separately by assuming a GGM where   $\mathbf{Y}_{i}|Z_{i}=k \sim N \left(\mu_{k}, \Sigma_{k} \right)$. In this paper we assume that $\forall$ $k$, $\mu_{k}=0$ .  In practice, this means that the data is assumed to be normalized by subtracting the mean.  Since $\mathbf{Y}_{i}$ is dependent on $Z_{i}$, we say that $Z_{i}$ represents the class that produced $\mathbf{Y}_{i}$ and we know  $\mathbf{Y}_{i}$ fully if we know which class $Z_{i}$ falls.  The density of each $\mathbf{Y}_{i}$ can be written as 

\begin{equation}
f_{\gamma}(\mathbf{y}_{i})=\sum^{K}_{k=1}\pi_{k}\varphi_{k}(\mathbf{y}_{i}|\Theta_{k})
\end{equation} where $\varphi(\mathbf{y}_{i}|\Theta_{k})$ denotes the density of Gaussian distribution with mean $0$ and inverse covariance covariance matrix $\Theta_{k}$; $f_{\mathbf\gamma}$  represents the ``incomplete''  mixture data density of the sample i.e $\mathbf{y} \sim f_{\gamma}$. We introduce the parameter set of mixture namely 

$$\Omega=\left\{\left\{\Theta_{k}\right\}^{K}_{k=1}| \Theta_{k}\succ 0, \quad k=1,...,K\right\}$$; $\Theta \succ 0$ indicates that $\Theta$ is positive-definite matrix, and 
$$J=\left\{\left\{\pi_{k}\right\}^{K}_{k=1}| \pi_{k}> 0, \quad k=1,...,K\right\}$$ where

\begin{equation}
\Gamma= \Omega \times J
\end{equation} denotes the parameter space with the true parameter defined as $\gamma_{0}=(\Theta_{0},\pi_{0}) \in \Gamma$.

In order to characterize the mixture model and estimate its parameters thereby recovering the underlying graphical structure from the data (seen as mixture of multivariate densities), several approaches may be considered. These approaches include graphical methods, methods of moments, minimum-distance methods, maximum likelihood \citep{Ruan11, Pan2009} and Bayesian methods \citep{bernardo2003bayesian, Biernacki2000}.  In our case we adopt the ML method.

\subsection{The penalized model-based likelihood}
We can now write the  likelihood of the incomplete data density as 

$$
L_{\mathbf{y}}(\gamma)=\prod^{n}_{i=1}\left[\sum^{K}_{k=1}\pi_{k}\varphi_{k}(\mathbf{y}_{i}| \Theta^{-1}_{k})\right]
$$ whose log-likelihood function is given by

\begin{equation}
l_{\mathbf{y}}(\gamma)=\sum^{n}_{i=1}\log f_{\gamma}(\mathbf{y}_{i})
\label{IC}
\end{equation}

The goal is to maximize the log-likelihood in (\ref{IC})  with respect to $\gamma$. Unfortunately,  a unique global maximum likelihood estimate does not exist because of the  permutation symmetries of the mixture subpopulation; \citep{Day1969, lindsay06}. Also the likelihood function of normal mixture models is not a bounded function on $\gamma$ as was put forward by \cite{Kiefer1956}. On the question of consistency of the MLE,  \cite{chanda1954}, \cite{Cramer1946}  focus on local ML estimation and mathematically investigate the existence of a consistent  sequence of local maximizers. These results are mainly based on Wald's technique \citep{Wald49} . \cite{Redner1981} later extended these results to establish the consistency of the MLE for mixture distributions with restrained or compact parameter spaces. It was proved that the MLE exists and it is globally consistent in a compact subset  $\hat{\Gamma}$ of $\Gamma$ that contains $\gamma_{0}$; i.e 

 $$ \mbox{given} \quad  \hat{\gamma}_{n}| l_{\mathbf{y}}(\hat{\gamma}_{n})= \underset{\gamma \in \hat{\Gamma}}{\mbox{max}} \quad l_{\mathbf{y}}(\gamma) , \quad \hat{\gamma}_{n} \rightarrow \gamma_{0} \mbox{ in  prob. for  n} \rightarrow \infty$$ 

 In addition to the degenerate nature of the likelihood  \cite{Kiefer1956}  on the set $\Gamma$, the``high dimensional, low sample size setting'' - where the number of observations $n$ is smaller that the number of nodes or features $p$ - is another complication. Estimating the parameters in the GGMM by maximizing criterion (\ref{IC})  is a complex one.  The penalized likelihood-based method \citep{Friedman01072008, Yuan01032007} is a promising approach to counter the degeneracy of $l_{\mathbf{y}}(\gamma)$ while keeping the parameter space $\Gamma$ unaltered. However, to make the PMLE work, one has to solve the problem of what kind of penalty functions  are eligible. We opt for a penalty function that guarantees  consistency  and also prevents the likelihood from degenerating under the multivariate mixture model. We assume that the penalty function $P: \Gamma\rightarrow \mathbb{R}^{+}_{0}$   satisfies:

\begin{equation}
\underset{|\Theta_{k}|\rightarrow \infty}{\lim} P(\Theta_{k})|\Theta_{k}|^{n}=0 \quad \forall k\in \left\{1,2,...,K\right\} \quad \forall n
\end{equation} where $|\Theta|$ denotes determinant of $\Theta$, $P(\Theta)=\exp(-\lambda ||\Theta||_{1})$.

This results in placing an $l_{1}$ penalty on the entries of the concentration matrix so that the resulting estimate is sparse and zeroes in this matrix correspond to conditional independency between the nodes similar to \citep{Meinshausen06}. Numerous  advantages result from this approach. First of all, the corresponding penalized likelihood is bounded and the penalized likelihood function does not degenerate in any point of the closure of parameter space $\hat{\Gamma}$ and therefore the existence of the penalized maximum likelihood estimator is guaranteed. Next, in the context of GGM, penalizing the precision matrix results in better estimate and sparse models are more interpretable and often preferred in application. 

We define  the $l_{1}$ penalized log-likelihood as:

 \begin{equation}
l^{p}_{{\mathbf{y}}}(\gamma)=l_{\mathbf{y}}(\gamma)-\lambda\sum^{K}_{k=1}||\Theta_{k}||_{1}
\label{pe}
\end{equation}   where $\lambda >0$ is a user-defined tuning parameter that regulates the sparsity level, $||\Theta||_{1}=\sum_{i, j}|\Theta_{ij}|$, $K$ is the number of mixing components assumed fixed. The hyperparameters $K$ and $\lambda$ determine the complexity of the model. The corresponding PMLE are defined as

\begin{equation}
\hat{\gamma}_{\lambda} =\underset{\gamma} {\mbox{argmax}} \quad l^{p}_{y}(\gamma) 
\label{pme}
\end{equation}

Our method penalizes all the entries of the precision matrix including the diagonal elements. We do this in order to avoid the likelihood to degenerate. To see this, consider a special case of a model consisting of two univariate normal mixtures $\pi_{1}\varphi(\mathbf{y},\sigma_{1})+ \pi_{2}\varphi(\mathbf{y},\sigma_{2})$. By letting $\sigma_{1} \rightarrow 0$ with other parameters remaining constant, the likelihood tends to infinity for values of $y=0$ ; i.e  the likelihood degenerates due to mixture formulation whereby a single observation mixture component with a decreasing variance on top of the observation explodes the likelihood. For that matter an $l_{1}$ penalty which does not  penalize the diagonal elements tend to  result in a degenerate ML estimator especially when $n \rightarrow \infty$.

\subsection{Consistency}
At this stage we want to characterize the solution obtained by maximizing (\ref{pe}).  The general theorem concerning the consistency of the MLE \citep{redner1980maximum, Wald49} can be extended to cover our type of penalized MLE. This is because if a likelihood function which yields a strong consistent estimate over a compact set is given, then our $l_{1}$ penalty would not alter the consistency properties. Consistency of the PMLE is given in theorem \ref{th1}. The latter uses results in \citep{Wald49} under the classical MLE over a compact set. The MEL version of  theorem \ref{th1} can be found in  \citep{redner1980maximum}. We define first a set of conditions upon which theorem 1 holds.

\begin{enumerate}
	\item[C1:]  Let the parameter space $\Gamma$ be compact set,  and $\bar{\Gamma}$ denotes the quotient topological space obtained from $\Gamma$ and suppose that $\bar{\Gamma}$ is any compact subset containing $\gamma_{0}$.
	\item[C2:] Let $B_{r}\left(\gamma\right)$ be the closed ball of radius $r$ about $\gamma.$ Then for any positive real number $r$, let:
						$$f_{\gamma}(\mathbf{y}, r)= \underset{\eta \in B_{r}\left(\gamma\right)}  {sup} f_{\gamma}(\mathbf{y}, \eta); \quad f^{*}_{\gamma}(\mathbf{y}, r)= max\left[1,f_{\gamma}(\mathbf{y}, r)\right]$$ Then
	for each $\gamma$ and for sufficiently small $r$
	$$  \int \log f^{*}_{\gamma}(\mathbf{y}, r)f_{\gamma_{0}}(\mathbf{y},r) < \infty$$ 
	\item[C3:]  
	$$\int |\log f_{\gamma_{0}}(\mathbf{y})|f_{\gamma_{0}}(\mathbf{y}) < \infty$$ 
	\item[C4:]  
	
	$$\int |\log f_{i}(\mathbf{y},\gamma_{i})|f_{j}(\mathbf{y},\gamma_{j}) < \infty \quad for\quad \gamma_{i} \in \Gamma_{i} \quad and \quad \gamma_{j} \in \Gamma_{j}$$ 
	
\end{enumerate}

\begin{theorem}
Suppose that the mixing distribution satisfy conditions (C1-4). Define ${|\gamma_{0}|}={||\pi_{0}||}_{2}+{||\Theta_{0}||}_{F}$. Suppose that $\pi_{k}$ is bounded away from zero, it follows that for a fixed $p$, the penalized likelihood solution $\hat{\gamma}_{\lambda_{n}}$ is  consistent in the quotient topological space $ \bar{\Gamma} $ i.e $\forall \quad \epsilon >0$
					$$\underset{n\rightarrow \infty} {\lim}  P\left(|\hat{\gamma}_{\lambda_{n}}-\gamma_{0}|> \epsilon \right)=0 $$
					\label{th1}
\end{theorem}

\begin{proof}
Let  the PMLE $\hat{\gamma}_{\lambda_{n}}$ and MLE $\hat{\gamma}_{n}$ be defined by

$$\hat{\gamma}_{\lambda_{n}} =\underset{\gamma} {\mbox{argmax}} \quad l^{p}_{n}(\gamma) $$ and

$$\hat{\gamma}_{n}={\mbox{argmax}} \quad l(\gamma)$$ where  
$$ l^{p}_{n}(\gamma)= l(\gamma)-\lambda_{n}\sum^{K}_{k=1}||\Theta_{k}||_{l_{1} \quad \forall \quad k \in \left\{1,...,K\right\}}$$ Then

$\forall \epsilon >0$ we have
\begin{eqnarray}
P(|\hat{\gamma}_{\lambda_{n}}-\gamma_{0}|> \epsilon)&=&P(|\hat{\gamma}_{\lambda_{n}}-\hat{\gamma}_{n}+\hat{\gamma}_{n}-\gamma_{0}|> \epsilon)\nonumber\\
&\leq& P(|\hat{\gamma}_{\lambda_{n}}-\hat{\gamma}_{n}|> \epsilon/2)+P(|\hat{\gamma}_{n}-\gamma_{0}|> \epsilon/2)\nonumber\\
\label{ineq}
\end{eqnarray}

Considering the  second inequality  on the RHS of (\ref{ineq}), we can write that
					$$\underset{n\rightarrow \infty} {\lim}  P\left(|\hat{\gamma}_{n}-\gamma_{0}|> \epsilon/2 \right)=0 $$
 This follows from theorem 5 of \cite{Redner1981}. Therefore it is sufficient to prove that 
					$$\underset{n\rightarrow \infty} {\lim}  P(|\hat{\gamma}_{\lambda_{n}}-\hat{\gamma}_{n}|> \epsilon/2)=0 $$
					
Suppose $l^{p}_{n}(\gamma)$ is bounded below by a function $l^{p}_{n,L}(\gamma)$  under the following assumptions:

\begin{enumerate}
	\item   There exists a neighborhood $\gamma_{0}$ of $\Gamma$  such that $l^{p}_{n,L}(\gamma)$ is continuously differentiable wit respect to parameters in $\gamma$ 
\item  $l^{p}_{n,L}(\gamma)$  converges (pointwise) to $l(\gamma)$ as $n \rightarrow \infty$

\end{enumerate}

 We define 

$$\hat{\gamma}_{\lambda_{n,L}} =\underset{\gamma} {\mbox{argmax}} \quad  l^{p}_{n,L}(\gamma)$$ 
Then  the followings hold:

$$\forall \delta >0  \quad \exists \quad n_{1} \in N \quad s.t. \quad \forall n>n_{1}\quad |\hat{\gamma}_{\lambda_{n,L}}-\hat{\gamma}_{n}|<\delta$$

Let $g_{n}$ be a function  such that $l^{p}_{n,L}<g_{n}$,  then
$$\hat{\gamma}_{g_{n}} =\underset{\gamma} {\mbox{argmax}} \quad  g_{n}(\gamma)$$ satisfies

$$\forall \delta >0  \quad \exists \quad n_{2} \in N s.t. \quad \forall n>n_{2}\quad |\hat{\gamma}_{g_{n}}-\hat{\gamma}_{n,L}|<\delta$$

Take $g_{n}=l^{p}_{n}$, then we can write
$$\forall \delta >0  \quad \exists \quad n_{3} \in N s.t. \quad \forall n>n_{3}\quad |\hat{\gamma}_{\lambda_{n}}-\hat{\gamma}_{n,L}|<\delta$$

Suppose $\delta=\frac{\epsilon}{4}$ and $n\geq max\left\{n_{1},n_{2}, n_{3}\right\}$, then

\begin{eqnarray}
P(|\hat{\gamma}_{\lambda_{n}}-\hat{\gamma}_{n}|> \frac{\epsilon}{2})&\leq&  P(|\hat{\gamma}_{\lambda_{n}}-\hat{\gamma}_{n,L}| > \frac{\epsilon}{4})+P(|\hat{\gamma}_{n,L}-\hat{\gamma}_{n}|>\frac{\epsilon}{4})\nonumber\\
&=& 0 \quad a.s \quad n \rightarrow \infty 
\end{eqnarray}
\end{proof}

\section{Penalized EM algorithm} \label{PMLE}

In order to  maximize the penalized likelihood function (\ref{pe}) we consider a penalized version of the EM algorithm of \cite{Dempster97} . To do that we first augment our data $\mathbf{Y}_{i}$ with $\mathbf{Z}_{i}$ so that the complete data associated with our model now becomes $\mathbf{C}_{i}=(\mathbf{Y}_{i},\mathbf{Z}_{i})$ and an EM algorithm iteratively maximizes, instead of the penalized observed log-likelihood $l^{p}_{{\mathbf{y}}}$ in (\ref{pe}), the quantity $Q(\gamma|\gamma^{(t)})$, the  conditional expectation of the penalized log-likelihood of the augmented data and $\Omega^{(t)}$ is the current value at iteration $t$.

Suppose $\mathbf{c}_{i} \sim h_{\mathbf{c}_{i}}(\gamma)$ i.e  $h_{\mathbf{c}_{i}}(\gamma)$ is the density of the augmented data $\mathbf{c}_{i}$.
Now the penalized log-likelihood of the augmented data  can be written as

\begin{eqnarray}
l^{p}_{\mathbf{c}}(\gamma)&=&\ln\left[h_{\mathbf{c}_{i}}(\gamma)\right]-\lambda\sum^{K}_{k=1}||\Theta_{k}||_{l_{1}}\nonumber\\
l^{p}_{\mathbf{c}}(\gamma)&=&\sum^{n}_{i=1}\ln \pi_{k}+\ln \phi_{k}(\mathbf{y}_{i}|\Theta^{-1}_{k})-\lambda\sum^{K}_{k=1}||\Theta_{k}||_{l_{1}}\nonumber\\
&=&\sum^{n}_{i=1}\sum^{K}_{k=1} 1_{\left\{Z_{i}=k\right\}}\left[\ln \pi_{k}+\ln \phi_{k}(\mathbf{y}_{i}|\Theta^{-1}_{k})\right]-\lambda_{n}\sum^{K}_{k=1}||\Theta_{k}||_{l_{1}}
\label{ll}
\end{eqnarray}

Note the indicator function $1_{\left\{Z_{(i)}=k\right\}}$ simply says that if you knew which component the observation $i$ came from, we would simply use its corresponding $\Theta_k$ for the likelihood. For illustration purpose, suppose we have $3$ observations and we are certain that the first two were generated by the Gaussian density $N(0,\Theta_{2})$ and the last came from $N(0,\Theta_{1})$, then we write the full log-likelihood as follows:
\begin{equation}
l_{\mathbf{Y}\mathbf{Z}}(\Theta)=l_{\mathbf{Y}_{1}}(\Theta_{2})+l_{\mathbf{Y}_{2}}(\Theta_{2})+l_{\mathbf{Y}_{3}}(\Theta_{1})
\end{equation}

\subsection{The E-step} 

 From (\ref{ll}), We compute the quantity $Q(\gamma|\gamma^{(t)})$ as follows
 
\begin{eqnarray}
Q(\gamma|\gamma^{(t)})&=&E_{\mathbf{Z}_{i}}\left[l_{\mathbf{Y}\mathbf{Z}}(\mathbf\gamma)-\lambda_{n}||\Theta||_{1}|y;\gamma^{(t)}\right]\nonumber\\
&=&\sum^{n}_{i=1}\sum^{K}_{k=1}\left[\ln \phi_{k}(\mathbf{y}_{i}|\Theta^{-1}_{k})+\ln \pi_{k}\right]E_{\mathbf{Z}_{i}}\left[ 1_{\left\{Z_{i}=k\right\}}|\mathbf{y}_{i};\gamma^{(t)}\right]-\lambda_{n} ||\Theta_{k}||_{1}\nonumber\\
&=&\sum^{n}_{i=1}\sum^{K}_{k=1}\left[\ln \phi_{k}(\mathbf{y}_{i}|\Theta^{-1}_{k})+\ln \pi_{k}\right]P\left(Z_{i}=k|\mathbf{y}_{i};\gamma^{(t)}\right)-\lambda_{n}||\Theta_{k}||_{1}\nonumber\\
&=&\sum^{n}_{i=1}\sum^{K}_{k=1}\left[\ln \phi_{k}(\mathbf{y}_{i}|\Theta^{-1}_{k})+\ln \pi_{k}\right]\omega^{(t)}_{ik}-\lambda_{n} ||\Theta_{k}||_{1}
\label{Emax}
\end{eqnarray}

The E-step actually consists of calculating $\omega_{ik}$, the  probabilities (condition on the data and $\gamma^{(t)}$) that  $\mathbf{Y}_{i}$'s originate from component $k$. It can also be seen as the responsibility that component $k$ takes for explaining the observation  $\mathbf{Y}_{i}$ and it tells us for which group an individual actually belongs. This is the soft K-mean clustering. Using Bayes theorem, we have:
\begin{eqnarray}
\omega^{(t)}_{ik}&=&P\left(Z_{i}=k|\mathbf{y}_{i},\gamma^{(t)}\right)\nonumber\\
&=&\frac{P(\mathbf{y}_{i}|Z_{i}=k;\gamma^{(t)})P(Z_{i}=k)}{\sum^{K}_{l=1}P(\mathbf{y}_{i}|Z_{i}=l;\gamma^{(t)})P(Z_{i}=l)}\nonumber\\
&=&\frac{ \phi^{(t)}_{k}(\mathbf{y}|\Theta^{-1}_{k})\pi^{(t)}_{k}}{\sum^{K}_{l=1} \phi^{(t)}_{l}(\mathbf{y}_{i}|\Theta^{-1}_{k})\pi^{(t)}_{l}}\nonumber\\
\end{eqnarray}

\subsection{The M-step}
The M-step for our mixture model can be split in to two parts, the maximization related to $\pi_{k}$ and the maximization related to $\Theta_{k}$.
\begin{enumerate}
	\item M-step for $\pi_{k}$:
	
For the maximization over $\pi_{k}$ we make use of the constraint that $\sum^{K}_{k=1}\pi_{k}=1$ i.e  $\pi_{K}=1-\sum^{K-1}_{k=1}\pi_{k}$ and $\pi_{k} >0$. It turns out that there is an explicit form for $\pi_{k}$.
Let $k_{0}\in \left\{1,...,K-1\right\}$. Then

\begin{equation}
\frac{\partial Q}{\partial \pi_{k_{0}}}=\sum ^{n}_{i=1}\left[\frac{\omega^{(t)}_{ik_{0}}}{\pi_{k_{0}}}-\frac{\omega^{(t)}_{iK}}{1-\sum^{K-1}_{k=1}\pi_{k}}\right]
\end{equation} 
 Setting $\frac{\partial Q}{\partial \pi_{k_{0}}}=0$, yields the following:

\begin{equation}
\omega^{(t)}_{.k_{0}}\sum^{K-1}_{k=1}\pi_{k}+\pi_{k_{0}}\omega^{(t)}_{.K}=\omega^{(t)}_{.k_{0}}
\label{equ}
\end{equation} It can be shown that a unique solution to (\ref{equ}) is

\begin{eqnarray}
\pi^{(t+1)}_{k_{0}}&=&\omega^{(t)}_{.k_{0}}/n   \nonumber\\
&=&\sum^{n}_{i=1}\omega^{(t)}_{ik_{0}}/n
\end{eqnarray} 

	\item M-step for $\Theta_{k}$:

Next, to maximize $(\ref{Emax})$ over $\Theta_{k}$, we only need the term that depends on  $\Theta_{k}$. The first thing we do here is to try to formulate the maximization problem for a mixture component to be similar to that for Gaussian graphical modeling with the aim of applying graphical lasso method.  Now from $(\ref{Emax})$,  for a specific cluster ${k_{0}}$, the term that depends on the cluster specific covariance matrix $\Theta_{k_{0}}$ is given by

\begin{eqnarray}
Q\left(\Theta_{k_{0}}\right)& =&\sum^{n}_{i=1}\omega^{(t)}_{ik_{0}}\ln \phi_{k_{0}}(\mathbf{y}_{i}|\Theta^{-1}_{k_{0}})-\lambda_{n}||\Theta_{k_{0}}||_{1}\nonumber\\
&=& \sum^{n}_{i=1}\omega^{(t)}_{ik_{0}}\left[\frac{1}{2}\ln|\Theta_{k_{0}}|-\frac{1}{2}\mathbf{y}^{'}_{i}\Theta_{k_{0}}\mathbf{y}_{i}\right]-\lambda_{n} ||\Theta_{k_{0}}||_{1}\nonumber\\
&=&\sum^{n}_{i=1}\frac{\omega^{(t)}_{ik_{0}}}{2} \ln|\Theta_{k_{0}}|-\frac{1}{2}tr\left(\sum^{n}_{i=1}\omega^{(t)}_{ik_{0}}(\mathbf{y}_{i}\mathbf{y}^{'}_{i})\Theta_{k_{0}}\right)-\lambda_{n} ||\Theta_{k_{0}}||_{1}\nonumber\\
&=& \frac{\omega^{(t)}_{.k_{0}}}{2}\left[\ln|\Theta_{k_{0}}|-tr\left(\tilde{S}_{k_{0}}\Theta_{k_{0}}\right)-\frac{2\lambda _{n}}{\omega^{(t)}_{.k_{0}}} ||\Theta_{k_{0}}||_{1}\right]\nonumber\\
&=& \frac{\omega^{(t)}_{.k_{0}}}{2}\left[\ln|\Theta_{k_{0}}|-tr\left(\tilde{S}_{k_{0}}\Theta_{k_{0}}\right)-\tilde{\lambda}_{n} ||\Theta_{k_{0}}||_{1}\right]
\label{max2}
\end{eqnarray} where 

\begin{equation}
\tilde{S}_{k_{0}}=\frac{\sum^{n}_{i=1}\omega^{(t)}_{ik_{0}}(\mathbf{y}_{i}\mathbf{y}^{'}_{i})}{\omega^{(t)}_{.k_{0}}}
\label{wei}
\end{equation} is the weighted empirical covariance matrix,  and

\begin{equation}
\hat{\Theta}_{k_{0}}=\underset{\Theta}{\mbox{argmax}}\left\{\ln|\Theta_{k_{0}}|-tr(\tilde{S}_{k_{0}}\Theta_{k_{0}})-\tilde{\lambda}_{n} ||\Theta_{k_{0}}||_{1}\right\}
 \label{max1}
 \end{equation} subject to the constraint that $\Theta_{k_{0}}$ is positive definite with $\tilde{\lambda}_{n}=\frac{2\lambda_{n}}{\omega^{(t)}_{.k_{0}}}$ .
\end{enumerate}
Therefore the maximization  of $\Theta_{k}$ consists of running the graphical lasso procedure \citep{Friedman01072008} for each cluster where each observation $\mathbf{Y}_{i}$ for  $\Theta_{k}$  gets a weight and the sampling covariance matrix $S_{k}$ is transformed to a weighted sampling covariance.   This  is a major innovation in our work where we formulate the Gaussian mixture modelling problem as a Gaussian Graphical modeling framework.

\section{Simulation and Real-data Exampple} 
We generate data from a two components mixture and consider two different schemes based on $\lambda_{n}$. We  study the consistency properties of the PMLE by allowing the sample size to grow. We subsequently applied our method to $2$ well known data ``Scor'  and ``CellSignal''  data. 
\begin{figure}[!ht]
	\centering
		\includegraphics[height =5cm, width =8cm]{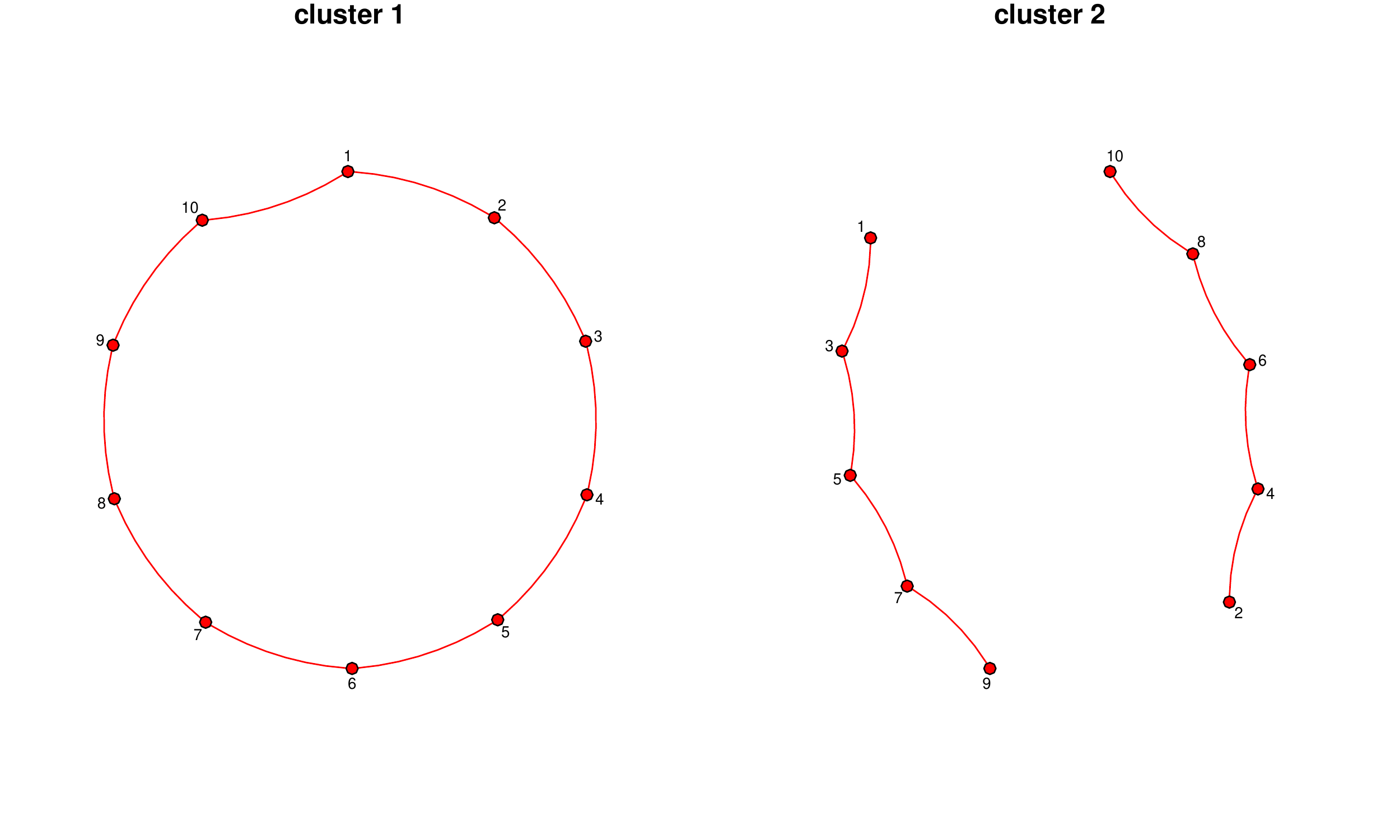}
	\caption{True graphical model of the 2 clusters}
	\label{sim}
\end{figure}

\subsection{Simulation}
We investigate the consistency  properties of the PMLE using our penalized EM algorithm described in section \ref{PMLE}. We simulate data $\mathbf{Y}_{1},...,\mathbf{Y}_{n}$ from two-component multivariate normal mixture models with probability (true mixture proportion) equals $0.5$ and inverse covariance  matrix $\Theta_{k}$ built according to  the following schemes. 

\begin{equation}
\Theta_{1}(i,j)=\left\{\begin{array}{l}1 \quad if \quad i=j\\
-0.4, \quad if |i-j|=1\\
0, \mbox{elsewhere}
\end{array}\right. 
\end{equation}

\begin{equation}
\Theta_{2}(i,j)=\left\{\begin{array}{l}1 \quad if \quad i=j\\
-0.4, \quad if |i-j|=2\\
0, \mbox{elsewhere}
\end{array}\right. 
\end{equation}


\begin{table}[!ht]
\centering
\begin{tabular}{ccccccrrr} 
	\hline
\textbf{\em Model}  & Bias(AD)/Frobenuis     &  $F_{1}$	score	  & TP &  FP&  Precison	&  Recall \\

\hline \hline  \textbf{\em n=100 }  \\
																																											
 Penalized     &  \\  																													
   $\pi$			&    AD=0.1125       &\\ 																																
	$\Theta_{1}$	&  F=1.7280 	&     	 0.555  & 5& 5&  0.5 & 0.625                &\\									                 																								       
	$\Theta_{2}$	&  F=1.6221   &        0.529  & 9& 15 &   0.375& 0.9&                  &\\				
							
	\hline \hline  \textbf{\em n=300 }  \\																																							
 Penalized   &  \\    																													
   $\pi$			&   AD=0.067       &\\ 																																
	$\Theta_{1}$	& F= 0.9702 &               0.5333& 8& 14& 0.3636& 1              &\\													               																								       
	$\Theta_{2}$	& F= 0.8432 &               0.5882& 10& 14&  0.4167 & 1              &\\

		\hline \hline  \textbf{\em n=800 }  \\
																																									
 Penalized   &  \\    																													
   $\pi$			&   AD=0.0625       &\\ 																																
	$\Theta_{1}$	& F=0.9279 &                 0.5882& 10&14& 0.4166& 1            &\\												               																								       
	$\Theta_{2}$	& F=0.4804 &               		 0.4705  & 8& 18 &0.3076& 1&                 &\\

		\hline \hline  \textbf{\em n=2000 }  \\
 Penalized   &  \\    																													
   $\pi$			&   AD=0.0263        &\\ 																																
	$\Theta_{1}$	& F=0.4170 &                 0.5925& 8&  11& 0.4210 & 1&              &\\												               																								       
	$\Theta_{2}$	& F=0.4465 &                0.625  & 10& 12 & 0.4545& 1&                 &\\

		\hline \hline  \textbf{\em n=5000 }  \\																																						
 Penalized   &  \\    																													
   $\pi$			&   AD=0.002        &\\ 																																
	$\Theta_{1}$	& F=0.3529 &                 0.6153& 8&10& 0.444& 1              &\\												               																								       
	$\Theta_{2}$	& F=0.2883 &              		0.6060  & 10& 13 & 0.4347& 1&                 &\\	
	
	\hline	
\end{tabular}
\caption{The  Absolute Deviation (AD), Frobenuis norm (F), the $F_{1}$ score, the True Positive (TP), the False Positive (FP), the Precision and the Recall of the PMLE for two-component mixture with $\lambda_{n} \propto \sqrt{ n \log p}$.}
\label{rep} 
\end{table}
The corresponding graphical model structures are depicted in Figure \ref{sim}.
For a fixed $p$, we consider two schemes one with  $\lambda_{n} \propto \sqrt{ n \log p}$  and the other with $\lambda_{n} \propto \sqrt{\log p}$, each with increasing sample sizes, $n=(100, 300, 800, 2000, 5000)$ to examine the consistency of the PMLEs.  In all cases, parameter estimation is achieved by maximizing the likelihood function via our penalized EM-algorithm and a model selection is performed inside the algorithm based on Extended Bayesian Information Criterion (EBIC), \citep{Chen01092008}. The results of our penalized EM-algorithm  approach are compared based on the two different schemes corresponding to different values of $\lambda_{n}$.

Due to the effect of label switching, we are not able to assign each parameter estimate correctly to the right class. As a result, the estimates $\left\{(\pi_{1},\Theta_{1}),  (\pi_{2},\Theta_{2})\right\}$  will be interchangeably represented.  We compute the Absolute Deviation (AD)  of the mixture proportions, and compare the Frobenuis norm of the difference between the true and estimated precision matrices for each cluster. In addition we compute the $F_{1}$ score, True positive (TP), False positive (FP), Precision and Recall for the PMLE.

\textbf{Example 1.}
We considered the simulated two-component multivariate normal mixture models above and   choose sequence of values of $\lambda_{n}$  such that     $c_{1} \sqrt{ n \log p} \leq  \lambda_{n} \leq  c_{2} \sqrt{ n \log p}$. On experimental basis we set $(c_{1},c_{2})=(0.1, 0.25)$. The performances of the penalized EM-algorithm corresponding to different sample sizes are presented in Table \ref{rep}. 

The results show that as  the sample size increases, the AD (for the mixture proportions) and the Frobenuis norms (for the precision matrices) decrease indicating the consistency of the PMLEs. At $n=5000$, the AD for the mixture proportion is almost $0$, indicating that our method has recovered precisely the true mixture distribution. Based on the EBIC criterion, we reported also the $F_{1}$ score, the True Positive (TP), the False Positive (FP), the Precision and the Recall of the PMLE. We recorded an overall improvement in the $F_{1}$ score as $n$ increases.

\textbf{Example 2.}
In this example, we again choose the same two-component multivariate Gaussian mixture model. In contrast to the model used in example 1, we have fixed the tuning parameter 
$\lambda_{n}$  such that     $c_{1} \sqrt{  \log p} \leq  \lambda_{n} \leq  c_{2} \sqrt{ \log p}$; $(c_{1},c_{2})$ remain unchanged.  The performances of the penalized EM-algorithm corresponding to different sample sizes are presented in  Table \ref{rep2}. We again observe a  decrease in both the Frobenuis norm  and the AD as $n$ increases even though  we suffer from a deficiency in the AD of $\pi$ for the case $n=800$. However  the AD is almost $0$ at $n=5000$. We note that this penalty decreases to $0$ faster and as result tends to produce full graph as can be seen in the higher value recorded for false positive.

\begin{figure}[bt]
	\centering
		\includegraphics[height =12cm, width =14cm]{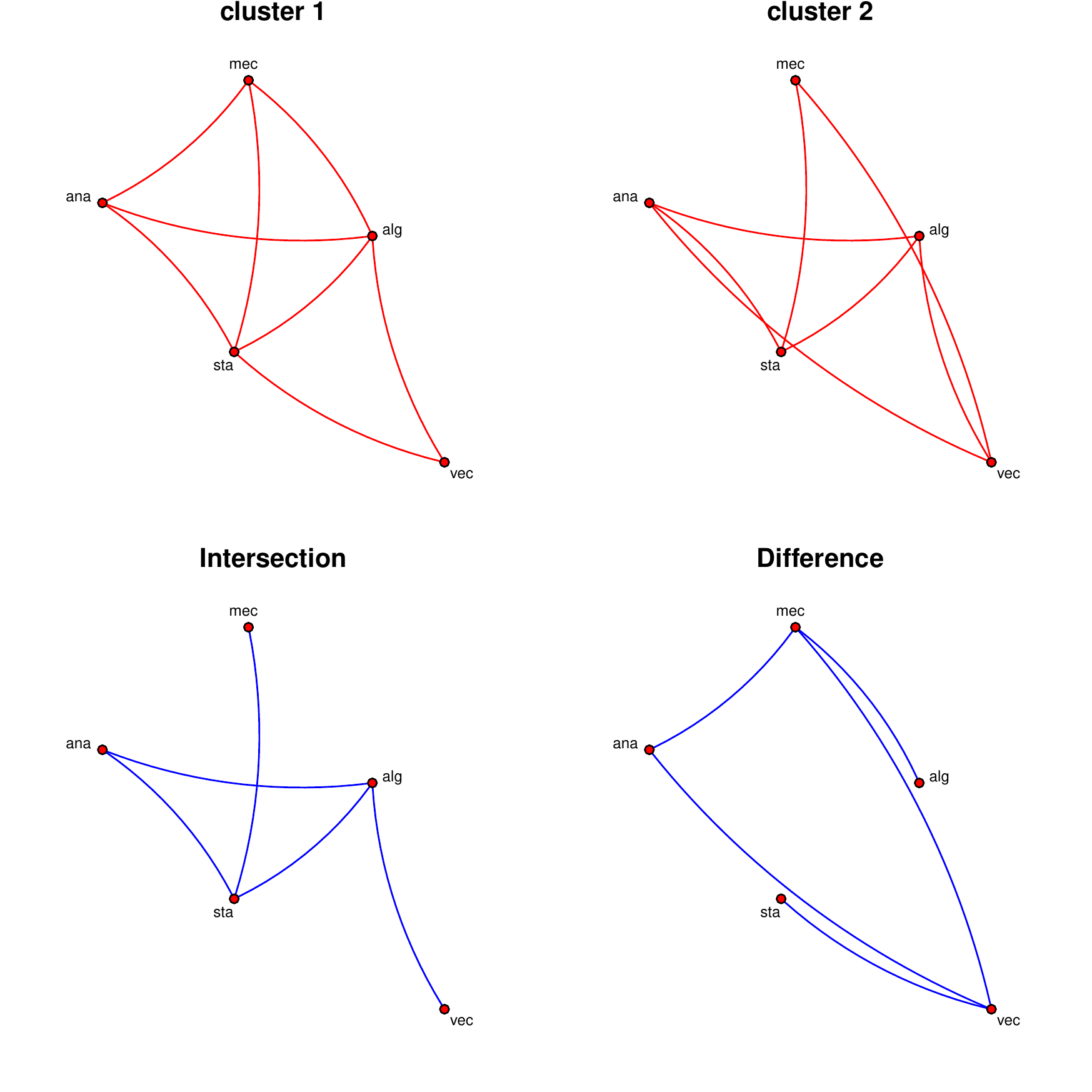}
	\caption{Graphical model of the 2 group of students}
	\label{scor}
\end{figure}

\begin{figure}[tb]
	\centering
		\includegraphics[height =10cm, width =14cm]{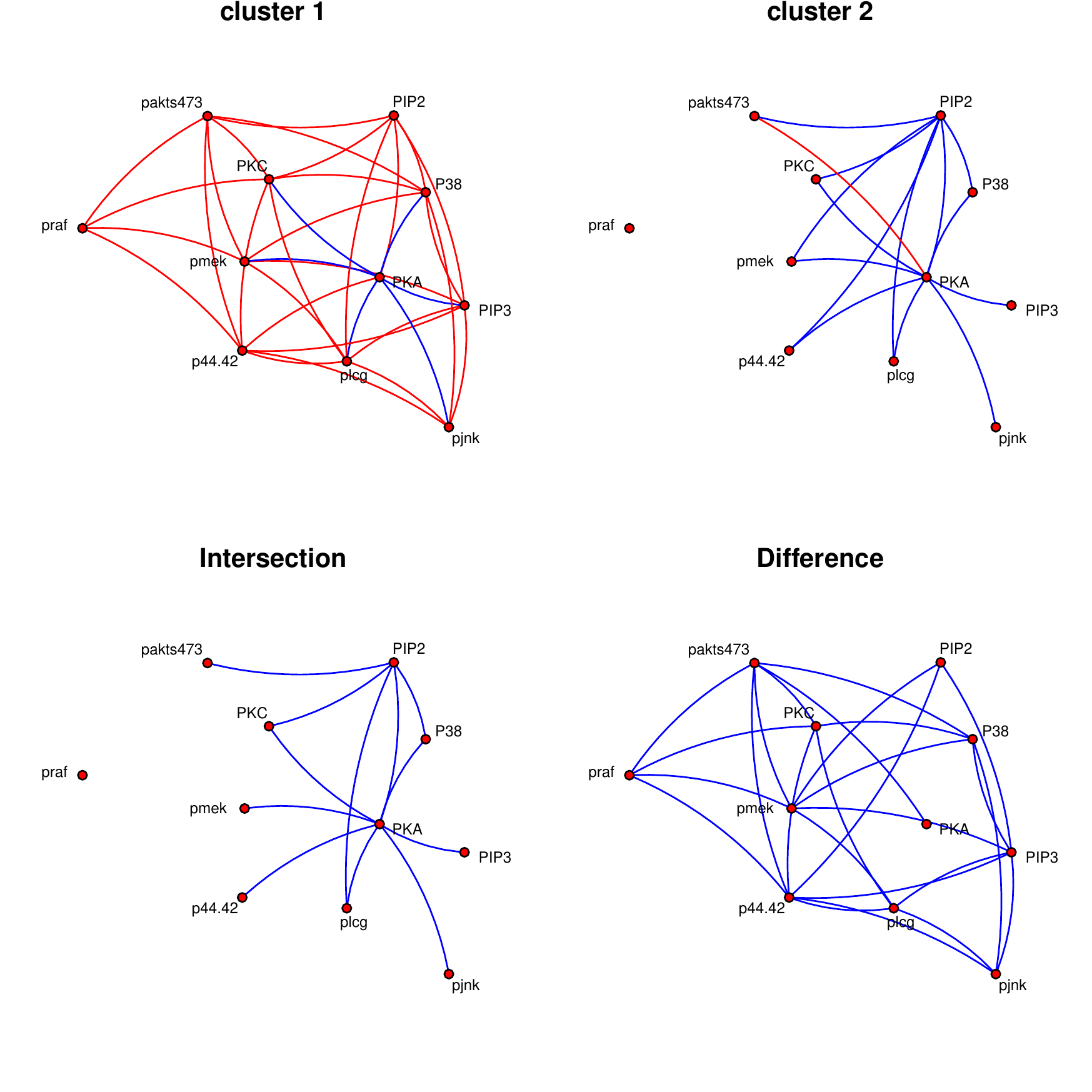}
	\caption{Graphical model of the CellSignal data based on first 500 cells}
	\label{signal}
\end{figure}


\begin{table}[!ht]
\centering
\begin{tabular}{ccccccrrr} 
	\hline
\textbf{\em Model}  & Bias(AD)/Frobenuis      &  $F_{1}$	score	  & TP &  FP&  Precison	&  Recall \\

\hline \hline  \textbf{\em n=100 }  \\
																																											
 Penalized     &  \\  																													
   $\pi$			&    AD=0.0307       &\\ 																																
	$\Theta_{1}$	&  F= 3.4081 	&     	 0.3446  & 10& 32&   0.2380  & 1               &\\									                 																								       
	$\Theta_{2}$	&  F= 3.4018  &         0.3181  & 7& 29 &  0.1944& 0.875&                  &\\				
							
	\hline \hline  \textbf{\em n=300 }  \\
																																							
 Penalized   &  \\    																													
   $\pi$			&   AD=0.0356        &\\ 																																
	$\Theta_{1}$	& F=1.0539&                0.3703& 10& 34&  0.2272 & 1              &\\												               																								       
	$\Theta_{2}$	& F=0.8657 &              	0.3137& 8& 35&  0.1860& 1              &\\		
																				
		\hline \hline  \textbf{\em n=800 }  \\																																						
 Penalized   &  \\    																													
   $\pi$			&   AD=0.0669       &\\ 																																
	$\Theta_{1}$	& F=0.6419 &                   0.3703& 10&34& 0.2272& 1            &\\												               																								       
	$\Theta_{2}$	& F=0.7605 &                  0.3018 & 8& 37 &  0.1777& 1&                 &\\	
									
		\hline \hline  \textbf{\em n=2000 }  \\
																																									
 Penalized   &  \\    																													
   $\pi$			&   AD=0.0312        &\\ 																																
	$\Theta_{1}$	& F=0.5081 &                 0.3168& 8&34& 0.1882& 1              &\\												               																								       
	$\Theta_{2}$	& F=0.4150 &                 0.3636& 10& 35 &  0.2222& 1&                 &\\	
		
		\hline \hline  \textbf{\em n=5000 }  \\																																							
 Penalized   &  \\    																													
   $\pi$			&  AD=0.0065        						&\\ 																																
	$\Theta_{1}$	& F=0.2771 &                0.3703& 10&34& 0.2272& 1              &\\												               																								       
	$\Theta_{2}$	& F=0.2857 &               	0.2692  & 7& 37& 0.1590& 0.875&                 &\\

	\hline	
\end{tabular}
\caption{The  Bias(AD), Frobenuis norm (F), $F_{1}$ score, True Positive (TP), False Positive (FP), Precision and Recall of the PMLE for two-component mixture with $\lambda_{n}\propto \sqrt{ \log p}$.}
\label{rep2} 
\end{table}

Comparing the 2 examples, we observe that the choice of $\lambda_{n}$ plays a strong hand in  parameter  and graph selection consistency  of the resultant networks. The consistency properties of the PMLEs was achieved in both cases but  our results indicates  that the overall performance of the asymptotic behavior of $\lambda_{n}\propto\sqrt{n\log p}$  is more satisfactory. Even though both penalty decrease to $0$ as $n$ increases,  $\lambda_{n}\propto\sqrt{n\log p}$  decreases slower resulting in a relatively sparser networks  as compared to  $\lambda_{n}\propto\sqrt{\log p}$.

\subsection{Real-data Examples}

\subsubsection{Open/Closed Book Examination Data}
As a simple example of a data set to which mixture models may be applied, we consider the ``scor'' data. This data can be found in the ``bootstrap'' package in R; type help(``scor'') in R for more details. 

This is a data on $88$ students who took examinations in $5$ subjects namely mechanics, vectors, algebra, analysis, statistics. Some where with open book and others with closed book. Mechanics and vectors were with closed book. 

We fit a two-mixture component to the data with a strong indication that there are two-groups of students each with similar subjects interest. We applied our PMLE algorithm to the data with $\lambda$ based on scheme 1.  The pattern of interaction among the two groups were depicted in Figure \ref{scor}. The network differences as well as similarities are also shown.  The results indicates that $61\%$ of students have similar subjects interest while $39\%$  falls in other group of interest. In one group, we observe no interactions between mechanics and analysis nor statistics and vectors while in the other group there are interactions.

\subsubsection{Analysis of cell signalling data}
We consider the application of our method on the flow cytometry dataset (cell signalling data) of \cite{Sachs22042005}. The data set contains flow cytometry of $p=11$ proteins measured on $n=7466$ cells; from which we selected the first $500$ cells. The CellSignal data were collected after a series of stimulatory cues and inhibitory interventions with cell reactions stopped at $15$ min after stimulation by fixation, to profile the effects of each condition on the intracellular signaling networks. Each independent sample in the data set is made up of quantitative amounts of each of the $11$ phosphorylated molecules, simultaneously measured from single cells. 

We again fit a two-mixture component to the data. The result of applying our PMLE algorithm to the data set using the first scheme is shown Figure \ref{signal}.  The result indicates that $90\%$ of the observation falls in one component whiles $10\%$ falls in the other cluster. We also display the differences and similarities in the two components. The following proteins interaction were seen to be present in each of the two components: ($pakts473, PIP2$), ($PKC, PIP2$), ($PKA, pjnk$), ($pmek, PKA$) to mention but few. Differences in the interaction occur among the following proteins: ($pakts473, praf$), ($PIP2, p44.42$), ($PKC,  plog$); see Figure \ref{signal}  for details.

\section{Conclusion}
We have developed a penalized likelihood estimator for  Gaussian  graphical mixture models. We impose an $l_{1}$ penalty on the precision matrix with extra condition preventing the likelihood not to degenerate.  The estimates were efficiently computed through a penalized version of the EM-algorithm.  By taking advantage of the recent development in Gaussian graphical models, we have implemented our method with the use of the graphical lasso algorithm. We have provided consistency properties for the penalized maximum likelihood estimator in Gaussian Graphical Mixture Model.  Our results indicate a better performance in  parameter consistency as well as in graph selection consistency for $\lambda_{n}= O(\sqrt{n\log p})$. Another interesting situation is when the order $K$, the number of mixture components in the model is unknown. This is a more practical problem than the one we have discussed and probably involves  simultaneous model selection. Thus we intend to continue our research in this direction and present the results in a future paper. 

\bibliographystyle{elsarticle-harv}
\bibliography{myrefs}







\end{document}